\documentclass[sn-mathphys-num, iicol]{sn-jnl}

\usepackage{graphicx}%
\usepackage{multirow}%
\usepackage{amsmath,amssymb,amsfonts}%
\usepackage{amsthm}%
\usepackage{mathrsfs}%
\usepackage[title]{appendix}%
\usepackage{xcolor}%
\usepackage{textcomp}%
\usepackage{manyfoot}%
\usepackage{booktabs}%
\usepackage{algorithm}%
\usepackage{algorithmicx}%
\usepackage{algpseudocode}%
\usepackage{listings}%

\usepackage{graphics}
\usepackage{hyperref}
\usepackage{url}

\usepackage{bbm}
\usepackage{tikz}
\usepackage{tikz-cd}
\usepackage{algorithm}
\usepackage{algpseudocode}
\usepackage{bm}
\usepackage{multirow}
\hypersetup{ hidelinks }
\usepackage{epstopdf}
\usepackage{booktabs}\let\cline\cmidrule

\theoremstyle{thmstyleone}%
\newtheorem{theorem}{Theorem}
\newtheorem{lemma}[theorem]{Lemma}
\newtheorem{corollary}[theorem]{Corollary}

\theoremstyle{thmstyletwo}%
\newtheorem{example}{Example}%

\theoremstyle{thmstylethree}%
\newtheorem{definition}{Definition}%

\begin{document}

\title[Mix-GENEO: A Flexible Filtration for Multiparameter Persistent Homology Detects Digital Images]{Mix-GENEO: A Flexible Filtration for Multiparameter Persistent Homology Detects Digital Images}


\author[1]{\sur{Jiaxing He}}\email{547337872@qq.com}

\author*[1]{\sur{Bingzhe Hou}}\email{houbz@jlu.edu.cn}

\author*[1,2]{\sur{Tieru Wu}}\email{wutr@jlu.edu.cn}

\author[3]{\sur{Yue Xin}}\email{179929393@qq.com}

\affil[1]{\orgdiv{School of Mathematics}, \orgname{Jilin University}, \orgaddress{ \city{Changchun}, \postcode{130012}, \country{P.~R. China}}}

\affil[2]{\orgdiv{School of Artificial Intelligence}, \orgname{Jilin University}, \orgaddress{\city{Changchun}, \postcode{130012}, \country{P.~R. China}}}

\affil[3]{\orgdiv{School of Mathematical Science}, \orgname{Heilongjiang University}, \orgaddress{\city{Harbin}, \postcode{150080}, \country{P.~R. China}}}


\abstract{Two important tasks in the field of Topological Data Analysis are building practical multifiltrations on objects and using TDA to detect the geometry. Motivated by the tasks, we build multiparameter filtrations by operators on images named multi-GENEO, multi-DGENEO and mix-GENEO, and we prove the stability of both the interleaving distance and multiparameter persistence landscape of multi-GENEO with respect to the pseudometric on bounded functions. We also give the estimations of upper bound for multi-DGENEO and mix-GENEO. In practical applications, we regard image as a discrete function space, and then we build multifiltrations on the discrete function space. Finally, we construct comparable experiment on MNIST dataset to demonstrate our bifiltrations are superior to 1-parameter filtrations including lower-star filtration and upper-star filtration. For instance, $6$ and $9$ can be distinguished by our bifiltrations, while they cannot be distinguished by 1-parameter filtrations. The experiment results demonstrate our bifiltrations have ability to detect geometric and topological differences of digital images.}

\keywords{Topological Data Analysis,  Multifiltration, Interleaving Distance,  Multiparameter Persistence Landscape}



\maketitle

\section{Introduction}\label{sec1}

The construction of filtration on images has always been an important issue in Topological and Geometric Data Analysis (TGDA). Currently, the sublevel set filtration to generate 1-parameter persistent homology is widely used. In the Figure \ref{num_3_PD}, we can see the lower-star filtration built on digit $3$  from the MNIST dataset by \cite{Dion-2012} only generate $H_{0}$ barcode $(0,+\infty]$ and $H_{1}$ barcode $(0,255]$, which are two meaningless signatures which only include the trivial messages. The persistence diagram is generated by Persim library \cite{scikittda2019}.  In the paper \cite{Bergo-2019}, the authors define group equivariant non-expansive operators (GENEOs), and in \cite{solomon2024convolutional}, the authors compute persistent homology on images by utilizing convolution operators. Compared to traditional sublevel set filtrations, their methods can improve accuracy to some extent, but our filter can significantly enhance accuracy. By applying specific operators to images, $H_{1}$ persistent homology obtained from the 1-parameter sublevel set filtration can identify the digits 1 and 3. However, this filtration cannot significantly identify the digits $1$ and $3$, $6$ and $9$ in the MNIST dataset by $H_{1}$ persistent homology and no one has  constructed multi-parameter filtrations by operators on images until now, which is a practical and effective approach. In this context, we propose three types of multifiltrations which are named multi-GENEO, multi-DGENEO, and mix-GENEO, which demonstrate superior performance in MNIST digit recognition. 
\begin{figure}[htbp]
	\centering
	\includegraphics[width=0.5\textwidth]{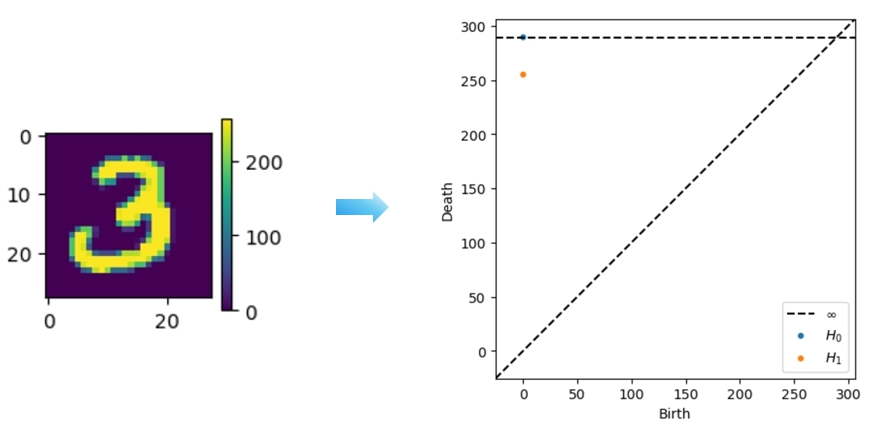}
	\caption{Persistence diagram $H_{0}$ and $H_{1}$ generated by lower-star filtration on the digit 3. }
	\label{num_3_PD}
\end{figure}

In the Figure \ref{num_3_module}, we show the persistent modules of $H_{0}$ and $H_{1}$ obtained by generating the mix-GENEO filtration on the digit 3. The multiparameter persistence module $H_{0}$ and $H_{1}$ provide more information about shapes of the images.

\begin{figure}[htbp]
	\centering
	\includegraphics[width=0.5\textwidth]{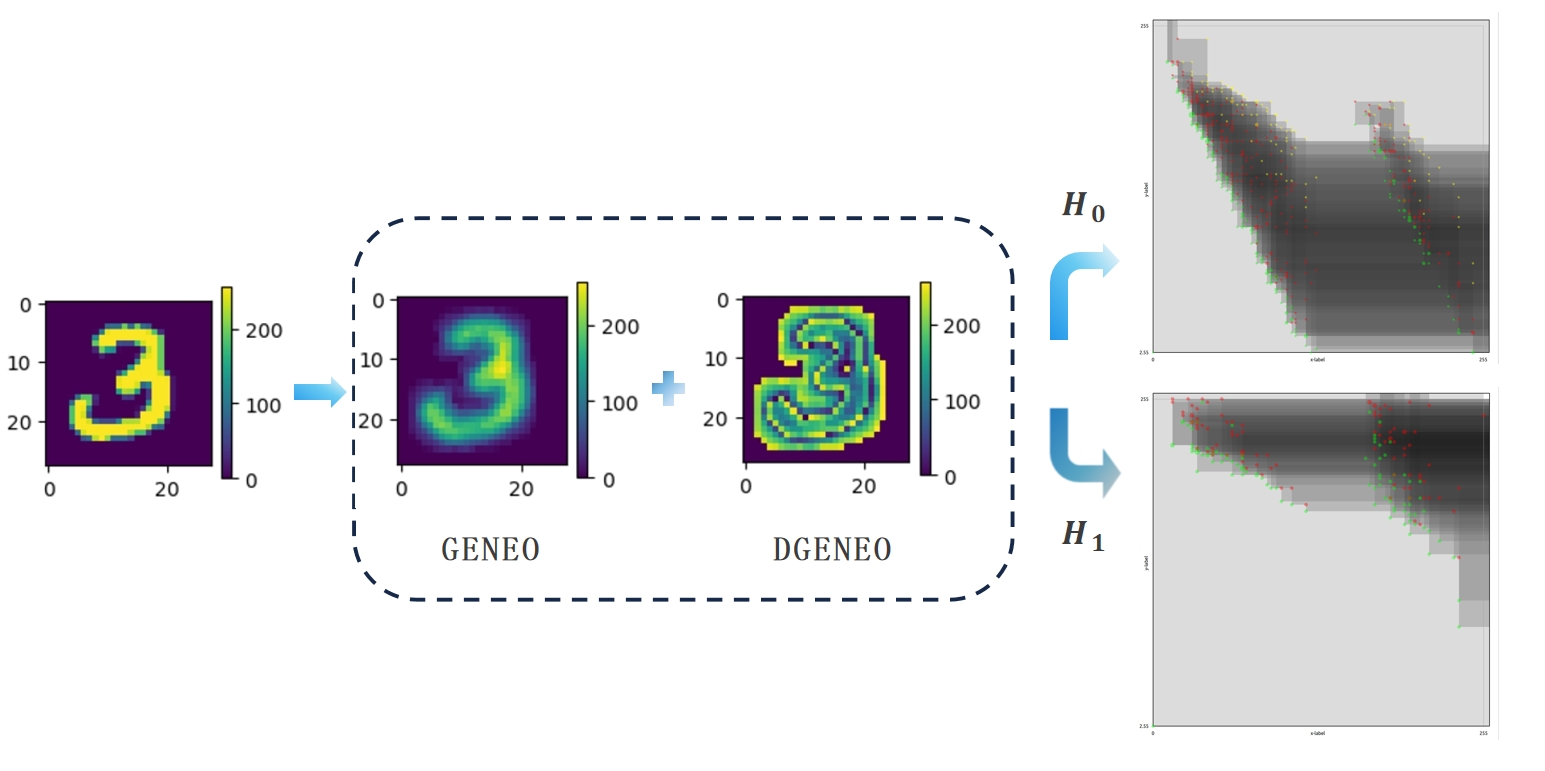}
	\caption{Multiparameter persistence module $H_{0}$ and $H_{1}$ generated by Mix-GENEO filtration on the digit 3. }
	\label{num_3_module}
\end{figure}

\subsection{Overview}
Topological and Geometric Data Analysis (TGDA) describes an emerging method to distinguish topological and geometry features combined with data analysis tools. While the history of TDA (Topological Data Analysis) could date back to the 1990's, the field has been developed rapidly in recent years, which leads to a rich of theoretical foundations \cite{Blum-2023, William-2015, Ulr-2020, Micheal-2015}, high efficient algorithms \cite{David-2006} and software \cite{RIVET-2020, Dion-2012}, and a board range of applications including medicine, ecology, materials science, deep learning and graphics \cite{Blum-2019, Bar-2021, Rick-2019, Hofer-2019, Hofer-2017, love2023topological, hacquard2022topologically}. 

A ubiquitous tool in TGDA is Persistence Homology (PH). The theory of PH studies the homological group of a family of topological spaces and its representation which is called persistence module, see \cite{Leo-2020, Ste-2015} for details. In probability theory, several authors have proposed estimators of fractal dimension defined in terms of minimum spanning trees and higher dimensional persistent homology \cite{Adams-2019, Ben-2020, Bob-2018,Sch-2019}. 

However, a single filtered space can not often adequately capture the structure of interest in the data. This leads one to consider multiparameter persistence. Multiparameter Persistence Homology (MPH) was first considered in \cite{Carls-2009}, in which they studied a multifiltration: a family of spaces parametrized along multiple geometric dimensions. The algebraic invariants of these multifiltrations are called multiparameter persistence modules. Unlike the single persistence, there is not an analogous complete discrete invariant for multiparameter module.

In \cite{Peter-2015} and \cite{Oli-2020}, the authors introduced the stable vectorization of the persistence diagram called persistence landscape and the stable vectorization of the multiparameter persistence module called multiparameter persistence landscape, respectively. Besides, another stable vectorization of the persistence diagram is the persistence image \cite{Adams-2017} which has been shown to produce favourable classification results when combined with machine learning methods. Multiparameter persistence image was introduced in \cite{Math-2020} which is suitable for machine learning and statistical frameworks.

Moreover, by using the geometric features of data extracted by PH and MPH as inputs for statistic techniques, one can provide new insights into the data. Persistence diagram could mark the parameter values for births and deaths of homological features. In a popular point of view, it is said that the long intervals represent the topological signal and the short intervals represent the noise. However, the authors \cite{Peter-2019} proved that persistent homology detects the curvature of disks which shows that the short intervals also encode geometric information. From \cite{Adams-2021}, persistent homology is a mathematically motivated out-of-the-box tool that one can use to summarize not only the global topology but also the local geometry of a wide variety of datasets.

Within the framework of 1-parameter persistent homology, there have been many proposals to build filtrations, including the removal of low density outliers \cite{Carls-2008}, filtering by  density functions and  kernel density functions \cite{Cha-2011, Bob-2017, Phil-2015}, measuring constructions by distances \cite{Anai-2019, Buch-2015, Chaz-2011, Gui-2013}, and subsampling \cite{Blum-2014}. However, there are several disadvantages for 1-parameter persistence. For instance, 1-parameter persistence is only determined by one single parameter, it is unable to distinguish small spatial features and large ones.

Several methods to construct multiparameter filtrations for points have been proposed such as the superlevel-rips bifiltration \cite{Carls-2009}, the multicover bifiltration \cite{Re-2023} and the rhomboid bifiltration \cite{Herb-2021}. These approaches can be found in \cite{Blum-2022} for details. In \cite{Math-2020}, the authors constructed a 2-parameter sublevel filtration from a pair by two images from a piece of human issue of a patient suffering from beast cancer.

\subsection{Motivation}

Many applications of 1-parameter persistent homology concern image analysis, where sublevel filtrations are often used. There is not yet a consensus on what the most natural or useful multifiltrations are for image analysis, but one promising idea is that a second persistence parameter can be used to thicken sublevel or superlevel sets, thereby introducing some sensitivity to the width of features that the ordinary sublevel and superlevel filtrations lack. One construction \cite{Yu-2022} along these lines is a framework to use morphological operations naturally form a multiparameter filtration to denoise. 

We would like to build a multifiltration of digital images to compute multiparameter homology, and then detect significant topological and geometric features from the multiparameter persistent landscape. 

The multiparameter landscape functions are sensitive to homological features of large, medium and small persistence. The landscapes also have the advantage of being interpretable since they are closely related to the rank invariant.

Frequently in topological data analysis, we need to consider several $\mathbb{R}$-valued functions
$$
\gamma_{i}:X\rightarrow \mathbb{R},\ i=1,...,n.
$$
It is equivalent to consider a function $\gamma:X\rightarrow{\mathbb{R}^{n}}$ on a topological space $X$ which gives rise to an $n$-$parameter\ sublevelset\ filtration\ \mathcal{S}(\gamma)$, defined by
$$
\mathcal{S}(\gamma)_{\boldsymbol{a}}=\{\ y\in{X}\ \mid\ \gamma(y)\leq \boldsymbol{a},\  \boldsymbol{a}\in \mathbb{R}^{n}\ \}.
$$

We want to explore the impact of different levels of filtration on multiparameter persistence module. In the paper \cite{Bergo-2019}, the authors defined group equivariant non-expansive operators(GENEOs) whose space is compact and convex with respect to the proper pseudometrics. Based on the stability, they described a simple strategy to select and sample operators and show how the operators can be used to perform machine learning. Also,  they provided a flexible way to select operators. GENEO can be viewed as Gaussian blur, and Laplace operator can be viewed as sharpening which can be thought of the subtraction of two different GENEOS called DOG. Combined with the definition of GENEO, we can use GENEO to construct n-parameter persistent filtrations which are named multi-GENEO, multi-DGENEO and mix-GENEO in the present paper.

To construct n-parameter filtration from a data set, we represent data as function. The following notations are from \cite{Bergo-2019}. Suppose that $X$ be a non-empty set and $\varPhi$ be a topological subspace of all bounded functions from $X$ to $\mathbb{R}$. Obviously, $\varPhi$ is naturally endowed with the topology induced by the distance $	D_{\varPhi}:={\Vert}\varphi_{1}-\varphi_{2}{\Vert}_{\infty} $.

Denote by $Homeo(X)$ the set of all homeomorphisms from X to X. For $g\in{Homeo(X)}$, if for every $\varphi\in{\varPhi}$,  $\varphi\circ{g}\in{\varPhi}$ and $\varphi\circ{g^{-1}}\in{\varPhi}$ , we say $g$ is a $\varPhi$-preserving homeomorphism. Denote by $Homeo_{\varPhi}(X)$ the set of all $\varPhi-$preserving homeomorphisms on X. Let $G$ be the subgroup of $Homeo_{\varPhi}(X)$, the pair  $(\varPhi,G)$ is called a $perception$ $pair$. Let $(\varPhi,G)$, $(\varPsi,H)$ be two perception pairs and $T:G\rightarrow{H}$ be a fixed homomorphism. If each linear operator $F:\varPhi\rightarrow\varPsi$ satisfies $F(\varphi\circ{g})=F(\varphi)\circ{T(g)}$ for every $\varphi \in \varPhi$, $g\in G$ is said to be a $group$ $equivariant$ $operator$  from $(\varPhi,G)$ to $(\varPsi,H)$. Moreover, the definition of GENEO is as follows,

\begin{definition}\cite{Bergo-2019}
	Assume that $(\varPhi,G)$, $(\varPsi,H)$ are two perception pairs and a homomorphism $T:G\rightarrow{H}$ has been fixed. If $F$ is a group equivariant operator from $(\varPhi,G)$ to $(\varPsi,H)$ with respect to $T$ and $F$ is non-expansive(i.e., $D_{\varPsi}(F(\varphi_{1}, F(\varphi_{2}))\leq{D_{\varPhi}(\varphi_{1},\varphi_{2})}$ for every $\varphi_{1},\varphi_{2}\in{\varPhi}$), then $F$ is called a $Group$ $Equivariant$ $Non-Expansive$ $Operator$ $(GENEO)$ associated with $T:G\rightarrow{H}$.
\end{definition}

In this paper, we could define multi-GENEO as follows. 

\begin{definition}\label{mul}
	A multi-GENEO filtration $\{\gamma_{i}\}_{i=1}^{n}$ is a multiparameter filtration defined by $\gamma_{i}=F^{i}(\varphi)$, where $\varphi \in \varPhi$, and each $F^{i}$ is a GENEO, $i=1,...,n$. A multi-DGENEO $\{\gamma_{i}\}_{i=1}^{n}$ is a multiparameter filtration defined by $\gamma_{i}=L^{i}(\varphi)=F^{1,i}(\varphi)-F^{2,i}(\varphi)$ where $\varphi\in \varPhi$, $F^{1,i}$ and $F^{2,i}$ are two elements in GENEO, $i=1,...,n$.
	Moreover, each $\gamma_{i}$ is chosen to be $F^{i}(\varphi)$ or $L^{i}(\varphi)$, we call $\{\gamma_{i}\}_{i=1}^{n}$ is a mix-GENEO.
\end{definition}

\subsection{Contributions}

In this paper we provide a flexible framework to build multiparameter filtrations on digital images. 

$\bullet$ We introduce three methods to build multiparameter filtrations called multi-GENEO, multi-DGENEO and mix-GENEO.

$\bullet$ We show the stability of both interleaving distance and multiparameter persistence landscape of multi-GENEO, and also provide bound estimates for both multi-DGENEO and mix-GENEO with respect to pseudometric for the subset of bounded functions.

$\bullet$ We conduct experiments on MNIST dataset and demonstrate our bifiltrations making sense to identify features of persistence modules by traditional machine learning methods, which  shows the ability of the multiparameter persistence homology to detect geometric and topological differences of digital images.

$\bullet$ We compare the results of lower-star filtration, upper-star filtration, multi-GENEO, multi-DGENEO and mix-GENEO by binary classifications and ten-classifications. In general, mix-GENEO performs the best.

To foster further developments at the intersection of multiparameter persistent homology and machine learning theory, we release our source code under:\url{https://github.com/HeJiaxing-hjx/Mix-GENEO/}.

\section{Background}

In this section, we will introduce some definitions and properties used in this paper. 

Let $\mathbb{Z}$ be the set of integers, $\mathbb{N}$ be the set of non-negative integers and $\mathbb{R}$ be the set of real numbers. Suppose that $\mathbb{K}$ is one of $\mathbb{Z}$, $\mathbb{N}$ and $\mathbb{R}$. For vectors $\boldsymbol{a}$, $\boldsymbol{b}$ in $\mathbb{K}^{n}$, there is a natural partial order on $\mathbb{K}^{n}$ by taking $(a_{1},...,a_{n})\leq (b_{1},...,b_{n})$ if and only if $a_{i} \leq b_{i}$ for all $1 \leq i \leq n$. Denote by $X$ a collection $\{X_{ \boldsymbol{a} } \}_{ \boldsymbol{a} \in \mathbb{R}^{n}}$ and denote by $\pi$ the collection of continuous maps $\pi_{\boldsymbol{a},\boldsymbol{b}}: X_{ \boldsymbol{a}} \rightarrow X_{ \boldsymbol{b}}$ such that the diagram commutes.
\begin{center}
	\begin{tikzcd}
		X_{ \boldsymbol{a}} \arrow{rd}{\pi_{\boldsymbol{a},\boldsymbol{c}}} \arrow{r}{\pi_{\boldsymbol{a},\boldsymbol{b}}} & X_{ \boldsymbol{b}} \arrow{d}{\pi_{\boldsymbol{b},\boldsymbol{c}}}\\
		& X_{ \boldsymbol{c}}
	\end{tikzcd}
\end{center}
 Denote by $\boldsymbol{Top}^{\mathbb{K}^{n}}$ the category whose objects are $(X,\pi)$ and whose morphisms are maps $f:(X,\pi) \rightarrow (Y, \tilde{\pi})$ which is a collection of all continuous maps $\{f_{\boldsymbol{a}}\}$ for all $\boldsymbol{a}\in \mathbb{K}^{n}$ such that $f_{\boldsymbol{a}}:X_{ \boldsymbol{a}} \rightarrow Y_{ \boldsymbol{a}}$ and the diagram commutes.
\begin{center}
	\begin{tikzcd}
		X_{ \boldsymbol{a}} \arrow{d}{f_{\boldsymbol{a}}} \arrow{r}{\pi_{\boldsymbol{a},\boldsymbol{b}}} & X_{ \boldsymbol{b}} \arrow{d}{f_{\boldsymbol{b}}}\\
		Y_{ \boldsymbol{a}} \arrow{r}{\tilde{\pi}_{\boldsymbol{a},\boldsymbol{b}}}         & Y_{ \boldsymbol{b}}
	\end{tikzcd}
\end{center}

\begin{example}
	Denote the sublevelset filtration $X_{\boldsymbol{t}}$  by
	$$
	\mathcal{S}(\gamma)_{\boldsymbol{t}}=\{\ y\in X\ |\ \gamma(y)\leq  \boldsymbol{t}\ \}
	$$ with natural inclusion $\pi_{\boldsymbol{t},\boldsymbol{s}}$, $\boldsymbol{t}\leq \boldsymbol{s} \in \mathbb{R}^{n}$. Then $(X,\pi)$ is one object of $\boldsymbol{Top}^{\mathbb{K}^{n}}$. Similarly, let  $Y_{\boldsymbol{t}}=\mathcal{S}(\gamma\circ f^{-1})_{\boldsymbol{t}}$ with natural inclusion  $\tilde{\pi}_{\boldsymbol{t},\boldsymbol{s}}$, $\boldsymbol{t}\leq \boldsymbol{s} \in \mathbb{R}^{n}$. Then $(Y,\tilde{\pi})$ is also an object in $\boldsymbol{Top}^{\mathbb{K}^{n}}$.  For a homeomorphism $\iota:X\rightarrow Y$, it can induce a morphism  $f:(X,\pi)\rightarrow (Y, \tilde{\pi})$, where $f_{\boldsymbol{t}}:X_{\boldsymbol{t}}\rightarrow Y_{\boldsymbol{t}}$ induced by $f_{t}(x)=\iota(x)$.
\end{example}

Let $M = \oplus_{\boldsymbol{a}\in \mathbb{K}^{n}} M_{\boldsymbol{a}}$, where $M_{\boldsymbol{a}}$ is a module. For any $\boldsymbol{a}\leq \boldsymbol{b}$, there is a homomorphism $\tau_{\boldsymbol{a},\boldsymbol{b}}: M_{\boldsymbol{a}} \rightarrow M_{\boldsymbol{b}}$ such that the following diagram commutes, 
\begin{center}
	\begin{tikzcd}
		M_{ \boldsymbol{a}} \arrow{rd}{\tau_{\boldsymbol{a},\boldsymbol{c}}} \arrow{r}{\tau_{\boldsymbol{a},\boldsymbol{c}}} & M_{ \boldsymbol{b}} \arrow{d}{\tau_{\boldsymbol{b},\boldsymbol{c}}}\\
		& M_{ \boldsymbol{c}}
	\end{tikzcd}
\end{center}
when $\boldsymbol{a}\leq \boldsymbol{b} \leq \boldsymbol{c}$. 

Denote by $\tau$ the collection of $\{\tau_{\boldsymbol{a},\boldsymbol{b}}\}$ for all $\boldsymbol{a} \leq \boldsymbol{b}$. Denote by $M^{\mathbb{K}^{n}}$ the category whose objects are $(M,\tau)$ and whose morphisms are maps  $h:(M,\tau)\rightarrow (N,\tilde{\tau})$ which is a collection of all continuous maps $\{h_{\boldsymbol{a}}\}$ for all $\boldsymbol{a}\in \mathbb{K}^{n}$ such that $h_{\boldsymbol{a}}:M_{ \boldsymbol{a}} \rightarrow N_{ \boldsymbol{a}}$ and the diagram commutes.
\begin{center}
	\begin{tikzcd}
		M_{ \boldsymbol{a}} \arrow{d}{h_{\boldsymbol{a}}} \arrow{r}{\tau_{\boldsymbol{a},\boldsymbol{b}}} & M_{ \boldsymbol{b}} \arrow{d}{h_{\boldsymbol{b}}}\\
		N_{ \boldsymbol{a}} \arrow{r}{\tilde{\tau}_{\boldsymbol{a},\boldsymbol{b}}}         & N_{ \boldsymbol{b}}
	\end{tikzcd}
\end{center}

Notice that homology can be viewed as a functor from $\boldsymbol{Top}^{\mathbb{K}^{n}}$ to $M^{\mathbb{K}^{n}}$. Define the functor $H:\boldsymbol{Top}^{\mathbb{K}^{n}}\rightarrow M^{\mathbb{K}^{n}}$ assigns to each object $(X,\pi)$ in $\boldsymbol{Top}^{\mathbb{K}^{n}}$ the object $(H(X), \pi_{*})$ in $M^{\mathbb{K}^{n}}$ and to each morphism $f\in \textrm{Mor}((X,\pi), (Y, \tilde{\pi}))$ in $\boldsymbol{Top}^{\mathbb{K}^{n}}$ the morphism $f_{*}\in \textrm{Mor}((H(X), \pi_{*}),(H(Y), \tilde{\pi}_{*}))$ in $M^{\mathbb{K}^{n}}$. Notice that $\mathbbm{1}_{*}=\mathbbm{1}$, $(g\circ f)_{*}= g_{*}\circ f_{*}$ and $(\tilde{\pi}_{\boldsymbol{a},\boldsymbol{b}}\circ f)_{*}= \tilde{\pi}_{(\boldsymbol{a},\boldsymbol{b})*}\circ f_{*}$. To see more details about homology theory, we refer to \cite{Hatcher}.

Next, we would like to introduce three pseudometrics $d_{\infty}$, $d_{I}$ and $d_{\lambda}^{(p)}$. Recall that an extended pseudometric on $X$ is a function $d:X\times X \rightarrow [0,\infty]$ with the following three properties:\\
(1)~$d(x,x)=0$, for all $x\in X$.\\
(2)~$d(x,y)=d(y,x)$, for all $x,y\in X$.\\
(3)~$d(x,z)\leq d(x,y) + d(y,z)$, for all $x,y,z\in X$ with $d(x,y), d(y,z)<\infty$.

An extended metric is an extended pseudometric $d$ with the additional property that $d(x,y)\neq{0}$ whenever $x\neq{y}$. In this paper, we refer to extended (pseudo)metrics simply as (pseudo)metrics.

The filtrations  of multi-GENEO we constructed are sublevelset filtrations. Let $\gamma^{X}:X\rightarrow{\mathbb{R}^{n}}$ be a sublevelset filtration function and $\textrm{hom}(\gamma^{X}, \gamma^{Y})$ be the set of all continuous functions $f:X\rightarrow{Y}$ such that $\gamma^{X}(p)\geq{\gamma^{Y}\circ{f(p)}}$ for every $p\in{X}$. We can define an n-parameter sublevelset filtration $S(\gamma)$ of any function $\gamma^{X}$.

For a function $\gamma:X\rightarrow{\mathbb{R}^{n}}$, let
$$
{\Vert}\gamma{\Vert}=\left\{
\begin{array}{rcl}
	\sup_{p\in{X}}{\Vert}\gamma(p){\Vert}_{\infty}       &   &   {\text{if}\   X\   \neq\    \emptyset}\\
	0  \quad \quad \      &   &   {\text{if}\   X\   =\    \emptyset}      
\end{array}
\right.
$$

Given $\gamma^{X}:X\rightarrow{\mathbb{R}^{n}}$ and  $\gamma^{Y}:Y\rightarrow{\mathbb{R}^{n}}$. Let 
$$
d_{\infty}(\gamma^{X}, \gamma^{Y}) = \inf_{h\in{\mathcal{H}}}{\Vert}\gamma^{X}-\gamma^{Y}\circ{h}{\Vert}_{\infty}
$$
where $\mathcal{H}$ is the set of homeomorphisms from $X$ to $Y$.

For $i\geq{0}$, we say that a pseudometric $d$ is $i$-stable for any topological spaces $X$, $Y$ and any functions $\gamma^{X}:X\rightarrow{\mathbb{R}^{n}}$, $\gamma^{Y}:Y\rightarrow{\mathbb{R}^{n}}$, we have
$$
d(H_{i}(\gamma^{X}),H_{i}(\gamma^{Y}))\leq{d_{\infty}(\gamma^{X}, \gamma^{Y})}.
$$
Moreover, we say a pseudometric is stable if it is $i$-stable for all $i\geq 0$.

For $\epsilon\in \mathbb{R}$, let $\vec{\boldsymbol{\epsilon}}\in \mathbb{R}^{n}$ denote the vector whose components are each $\epsilon$. Write $(\cdotp)(\vec{\boldsymbol{\epsilon}}):  M^{\mathbb{R}^{n}}\rightarrow M^{\mathbb{R}^{n}}$ simply as $(\cdotp)(\epsilon)$. Define $\tau_{\boldsymbol{a},\boldsymbol{a}+\vec{\boldsymbol{\epsilon}}}$ to be the $\epsilon$-$transition$ $morphism$ $\varphi_{M}^{\epsilon}: M_{\boldsymbol{a}}\rightarrow  M_{\boldsymbol{a}+\vec{\boldsymbol{\epsilon}}}$  for all $\boldsymbol{a}\in \mathbb{R}^{n}$. Simply write $M(\epsilon)= M_{\boldsymbol{a}+\vec{\boldsymbol{\epsilon}}}$. Two n-modules M and N are said to be $\epsilon$-$interleaved$ if there exist morphisms $f: M\rightarrow N(\epsilon)$ and $g: N\rightarrow M(\epsilon)$ such that 
\begin{align*}
	g(\epsilon)\circ f=\varphi_{M}^{2\epsilon},\ f(\epsilon)\circ g=\varphi_{N}^{2\epsilon}.	
\end{align*}
Here, we call $f$ and $g$ $\epsilon$-$interleaving$ $morphisms$.

Define the interleaving distance $d_{I}:M\times N\rightarrow{[0,\infty)}$ as follows:
\begin{align*}
&d_{I}(M, N)\\
 =& \inf\{\epsilon\in{[0,\infty)}\mid M\ \text{and}\ N\ \text{are}\ \epsilon-\text{interleaved}\}.
\end{align*}
The above $d_{I}$ is the same as the definition in \cite{Micheal-2015}, and the stability of $d_{I}$  is also given in \cite{Micheal-2015}.  

\begin{theorem} [\cite{Micheal-2015}]
	$d_{I}$ is stable.
\end{theorem}

Multiparameter persistence landscape proposed in \cite{Oli-2020} is a stable representation with respect to the interleaving distance and persistence weighted Wasserstein distance. The author also provided examples and statistical tests to demonstrate a range of potential applications which is convenient to be used.

Let $M\in{M^{\mathbb{R}^{n}}}$. Consider the function  $\lambda:\mathbb{N}\times\mathbb{R}^{n}\rightarrow \mathbb{R}$,
\begin{align*}
\lambda(k,\boldsymbol{x})= \sup\{\varepsilon\geq{0}: \ &\beta^{\boldsymbol{x-h},\boldsymbol{x+h}}\geq{k} \ \text{for all} \\
& \boldsymbol{h}\geq\boldsymbol{0}\ \text{with}\ {\Vert}\boldsymbol{h}{\Vert}_{\infty}\leq\epsilon \},
\end{align*}
where $\beta^{\boldsymbol{a},\boldsymbol{b}}=\textrm{dim}(\textrm{Im}(M_{a}\rightarrow{M_{b}}))$ is considered as the corresponding Betti number for $\boldsymbol{a}\leq\boldsymbol{b}$.  The multiparameter persistence landscape of $M$ is the set of such function $\lambda(k,\boldsymbol{x})$ which describes the maximal radius over which k features persist in every (positive) direction through $\boldsymbol{x}$ in the parameter space.

Let $M$, $N$ be multiparameter persistence modules. The $p$-$landscape$ distance $d_{\lambda}^{(p)}(M,N)$ between $M$ and $N$  is defined by,
$$
d_{\lambda}^{(p)}(M,N)={\Vert}\lambda(M)-\lambda(N){\Vert}_{p}
$$
where ${\Vert}\cdot{\Vert}$ is $L^{p}\text{-}norm$ for the $\mathbb{R}$-valued functions on $\mathbb{N}\times\mathbb{R}^{n}$.

\begin{theorem}[\cite{Oli-2020}]
	Let $M$, $N\in{M^{\mathbb{R}^{n}}}$ be multiparameter persistence modules, then the $\infty-landscape$ distance of the multiparameter persistence landscapes is bounded by the interleaving distance $d_{I}$, i.e.
	\begin{equation*}
		d_{\lambda}^{(\infty)}(M,N) \leq d_{I}(M,N).
	\end{equation*}
\end{theorem}

We would like to introduce lower-star filtration and upper-star filtration, which are both 1-parameter filtrations. Let $K$ be a triangulation of a compact 2-manifold without boundary $\mathbb{M}$. Let $h:\mathbb{M}\rightarrow \mathbb{R}$ be a function that is linear on every triangle. The function is defined, consequently, by its value at the vertices of $K$. We will assume that $h(u)\neq h(v)$ for all vertices $u\neq v\in K$. It is common to refer to h as the height function. In a simplicial complex, the natural concept of a neighborhood of a vertex $u$ is the $star$, $\text{St}u$, that consists of $u$ together with the edges and triangles that share $u$ as a vertex. Since all vertices have different heights, each edge and triangle has a unique lowest and a unique highest vertex. We can partition the simplicies of the star into lower and upper stars,
\begin{definition}[\cite{zomorodian2005topology}]
	The $lower$ $star$ $\b{\rm{St}}u$ and $upper$ $star$ $\bar{\rm{St}}u$ of vertex $u$ for the height function $h$ are
	\begin{align*}
		\b{\rm{St}}u = \{\sigma\in \text{Stu}~|~h(v)\leq h(u),\forall \text{vertices}~ v\in \sigma\},
	\end{align*}
	\begin{align*}
		\bar{\rm{St}}u = \{\sigma\in \text{Stu}~|~h(v)\geq h(u),\forall \text{vertices}~ v\in \sigma\},
	\end{align*}
\end{definition} 
These subsets of the star contain the simplices that have $u$ as their highest or their lowest vertex, respectively. And we may partition $K$ into a collection of either lower or upper stars, $K=\cup_{u}\b{\rm{St}}u = \cup_{u}\bar{\rm{St}}u$. Each partition give us a filtration. Suppose we sort the $n$ vertices of $K$ in order of increasing height to get the sequence $u^{1}$, $u^{2}$, $\ldots$, $u^{n}$, $h(u^{i})< h(u^{j})$, for all $1\leq i < j \leq n$. We then let $K^{i}$ be the union of the first $i$ lower stars, $K^{i} = \cup_{1\leq i < j \leq n}\b{\rm{St}}u^{j}$. Each simplex $\sigma$ has an associated vertex $u^{i}$, and we call the height of that vertex the $birth$ $time$ $h(\sigma)=h(u^{i})$ of $\sigma$. The subcomplex $K^{i}$ of $K$ consists of the $i$ lowest vertices together with all edges and triangles connecting them. Clearly, the sequence $K^{i}$ defines a filtration of $K$. We may define another filtration by sorting in decreasing order and using upper stars.

\section{Stability and Representation}

In this section, we will show the stability and the bound estimates with respect to both the interleaving distance and multiparameter persistence landscape of multi-GENEO,  multi-DGENEO and mix-GENEO persistence module. we will also show the filtrations of multi-GENEO, multi-DGENEO and mix-GENEO on discrete function space.

\subsection{Stability for Multi-GENEO}
Consider $F$ as an element in the n copies of GENEO written as $F=(F_{1}, F_{2}, ..., F_{n})\in{\oplus_{i=1}^{n}GENEO}$.

\begin{theorem}\label{mul-G}
	Suppose that $X$ is a non-empty space, $\varphi_{k} \in{\varPhi}$ are the bounded functions on X for $k=1,2$, and $F=(F^{1}, F^{2}, ..., F^{n})$ is an element in the n copies of GENEO. Let $V(F(\varphi_{k}))$ be the multiparameter persistence module of multi-GENEO. The filtration of multi-GENEO for $1\leq{i}\leq{n}$ can be obtained written as $F(\varphi_{k})$, and then
	\begin{align*}
		&\sup_{F}d_{\lambda}^{(\infty)}(V(F(\varphi_{1})),V(F(\varphi_{2}))) \\
		\leq& \sup_{F}d_{I}(V(F(\varphi_{1})),V(F(\varphi_{2})))\\ \leq& \inf_{g\in{G}}{\Vert}\varphi_{1}-\varphi_{2}\circ{g}{\Vert}_{\infty}\\ \leq&  D_{\varPhi}(\varphi_{1}, \varphi_{2}),
	\end{align*}
	where $G$ is a subgroup of $\textrm{Homeo}_{\varPhi}(X)$.
\end{theorem}

\begin{proof}	
	For every $F\in{\oplus_{i=1}^{n}GENEO}$, every $g\in{G}$ and $\varphi_{1}$, $\varphi_{2}\in{\varPhi}$, we have that
	\begin{align*}
		&d_{I}(V(F(\varphi_{1})),V(F(\varphi_{2})))\\
		=&d_{I}(V(F^{1}(\varphi_{1}),...,F^{n}(\varphi_{1})), \\
		&\quad V(F^{1}(\varphi_{2}),...,F^{n}(\varphi_{2})))\\ 
		=&d_{I}(V(F^{1}(\varphi_{1}),...,F^{n}(\varphi_{1})), \\
		&\quad   V(F^{1}(\varphi_{2})\circ T(g),...,F^{n}(\varphi_{2})\circ T(g)))\\
		=&d_{I}(V(F^{1}(\varphi_{1}),...,F^{n}(\varphi_{1})), \\
		&\quad   V(F^{1}(\varphi_{2}\circ g),...,F^{n}(\varphi_{2}\circ g)))\\
		\leq&D_{\varPsi}(V(F^{1}(\varphi_{1}),...,F^{n}(\varphi_{1})), \\
		&\quad   V(F^{1}(\varphi_{2}\circ g),...,F^{n}(\varphi_{2}\circ g)))\\
		=&{\Vert}(F^{1}(\varphi_{1}-\varphi_{2}\circ g), ...,F^{n}(\varphi_{1}-\varphi_{2}\circ g)){\Vert}_{\infty}\\
		=&\mathop{max}\limits_{i}{\Vert}F^{i}(\varphi_{1}-\varphi_{2}\circ g){\Vert}_{\infty}\\
		\leq&{\Vert}\varphi_{1}-\varphi_{2}\circ{g}{\Vert}_{\infty}\\
		=&D_{\varPhi}(\varphi_{1},\varphi_{2}\circ{g}).
	\end{align*}
	The second equality follows from the invariance of multiparameter persistent homology under action of $\textrm{Homeo}_{\varPhi}(X)$, the third equality and the seventh inequality follow from that each $F^{i}$ is a GENEO. The fourth inequality follows from the stability of multiparameter persistent homology while the sixth equality follows from the definition of the metric ${\Vert}\cdot{\Vert}_{\infty}$.
	Since $\varphi_{1}$, $\varphi_{2}$, $g$ are arbitrary chosen and $F$ is an element in the n copies of GENEO, we get
	\begin{align*}
		&\sup_{F}d_{I}(V(F(\varphi_{1})),V(F(\varphi_{2})))\\
		\leq& \inf_{g\in{G}}{\Vert}\varphi_{1}-\varphi_{2}\circ{g}{\Vert}_{\infty}\\
		\leq&  D_{\varPhi}(\varphi_{1},\varphi_{2}).
	\end{align*}
	Furthermore, we have that 
	\begin{align*}
		&\sup_{F}d_{\lambda}^{(\infty)}(V(F(\varphi_{1})),V(F(\varphi_{2})))\\ \leq& \sup_{F}d_{I}(V(F(\varphi_{1})),V(F(\varphi_{2})))\\ \leq& \inf_{g\in{G}}{\Vert}\varphi_{1}-\varphi_{2}\circ{g}{\Vert}_{\infty}\\ \leq & D_{\varPhi}(\varphi_{1},\varphi_{2}).
	\end{align*}
	Then we obtain the stability of the $\infty$-landscape distance of the multiparameter persistence landscapes.
	
\end{proof}

\begin{lemma}\label{mul-L}
	Suppose that $X$ is a non-empty space, $\varphi_{k} \in{\varPhi}$ are the bounded functions on X for $k=1,2$, and $L^{i}(\varphi)=F^{1,i}(\varphi)-F^{2,i}(\varphi)$, for which  $F^{1,i}(\varphi)$ and $F^{2,i}(\varphi) $ are two elments in the GENEOs, $i=1,...,n$.
	Let $L=(L^{1}, L^{2}, ..., L^{n})$ and $V(L(\varphi_{k}))$ be the multiparameter persistence module of multi-DGENEO. The filtration of multi-DGENEO for $1\leq{i}\leq{n}$ can be obtained written as $L^{i}(\varphi_{k})$, and then
	\begin{align*}
		&\sup_{L}d_{\lambda}^{(\infty)}(V(L(\varphi_{1})),V(L(\varphi_{2}))) \\
		\leq& 	\sup_{L}d_{I}(V(L(\varphi_{1})),V(L(\varphi_{2})))\\ \leq& 2\inf_{g\in{G}}{\Vert}\varphi_{1}-\varphi_{2}\circ{g}{\Vert}_{\infty} \\ \leq& 2D_{\varPhi}(\varphi_{1},\varphi_{2}),
	\end{align*}
	where $G$ is a subgroup of $Homeo_{\varPhi}(X)$.
\end{lemma}

\begin{proof}	
	As the same of the calculation in Theorem \ref{mul-G}, we have 
	\begin{align*}
		&d_{I}(V(L(\varphi_{1})),V(L(\varphi_{2})))\\	
		\leq&D_{\varPsi}(V(L^{1}(\varphi_{1}),...,L^{n}(\varphi_{1})), \\
		&\quad \ \ V(L^{1}(\varphi_{2}\circ g),...,L^{n}(\varphi_{2}\circ g)))\\
		=&\max\limits_{i}{\Vert}L^{i}(\varphi_{1}-\varphi_{2}\circ g){\Vert}_{\infty}\\
		\leq&2{\Vert}\varphi_{1}-\varphi_{2}\circ{g}{\Vert}_{\infty}\\
		=&2D_{\varPhi}(\varphi_{1},\varphi_{2}\circ{g}).
	\end{align*}
	Since $\varphi_{1}$, $\varphi_{2}$, $g$ are arbitrary chosen and $F$ is considered as an element in the n copies of GENEO, the conclusion is obtained.
\end{proof}

\begin{corollary}\label{mul-M}
	Let $V(M(\varphi_{k}))$ be the multiparameter persistence module of mix-GENEO, $k=1,2$. Then
	\begin{align*}
		&\sup_{M}d_{\lambda}^{(\infty)}(V(M(\varphi_{1})),V(M(\varphi_{2}))) \\
		\leq& \sup_{M}d_{I}(V(M(\varphi_{1})),V(M(\varphi_{2})))\\ \leq& 2\inf_{g\in{G}}{\Vert}\varphi_{1}-\varphi_{2}\circ{g}{\Vert}_{\infty}\\ \leq&  2D_{\varPhi}(\varphi_{1},\varphi_{2}).
	\end{align*}
\end{corollary}

\begin{proof}
	As the same of the proof in Lemma  \ref{mul-L}, we can use the Definition \ref{mul} to get the conclusion. 
\end{proof}

\subsection{Representation on Discrete Function Spaces}

Similar to the representation on discrete function spaces of 1-parameter GENEO construction in \cite{Bergo-2019}, we can construct filtrations of multi-GENEO. Let $\{\sigma^{i}\}_{i=1}^{n}$ be a sequence of positive numbers and  $\{\tau^{i}\}_{i=1}^{n}$ be a sequence of real numbers.  We consider the $\{g_{\tau^{i}}\}_{i=1}^{n}$ for each $g_{\tau^{i}}:\mathbb{R}\rightarrow\mathbb{R}$ is a 1-dimensional Gaussian function with width $\sigma^{i}$ and center $\tau^{i}$,
$$g_{\tau^{i}}(t):=\exp\left\{{-\frac{t-\tau^{i}}{2({\sigma^{i}})^2}}\right\}.$$

For a positive integer k, let set $S$ be the set of such $2k$-tuples $(a_{1},\tau_{1},...,a_{k},\tau_{k})\in{\mathbb{R}^{2k}}$ that $\sum_{j=1}^k{a_{j}^{2}}=\sum_{j=1}^k{\tau_{j}^{2}}$. Let $p=(p_{1},...,p_{n})$ and $p_{i}=(a^{i}_{1},\tau^{i}_{1},...,a^{i}_{k},\tau^{i}_{k})\in{S}$ for $i=1,...,n$. Define the function $G_{p}=(G_{p}^{1},...,G_{p}^{n})$ by 
$$G^{i}_{p}(x,y):=\sum_{j=1}^k{a_{j}g_{\tau^{i}_{j}} (\sqrt{x^{2}+y^{2} })}.$$

Define the convolutional operator $F^{i}_{p}$ as follows. For each continuous map $\varphi:\mathbb{R}^{2}\rightarrow\mathbb{R}$ with compact support, $F^{i}_{p}(\varphi):\mathbb{R}^{2}\rightarrow\mathbb{R}$ is the continuous map with compact support in the following form,
$$
F^{i}_{p}(\varphi)(x,y):=\int_{\mathbb{R}^{2}}\varphi(\alpha,\beta)\frac{G^{i}_{p}(x-\alpha,y-\beta)}{{\Vert}{G^{i}_{p}{\Vert}_{L^{1}}}}d\alpha d\beta.
$$
Then the operator $F^{i}_{p}$ is a GENEO with respect to the group $I$ of Euclidean plane isometries. One can see that $\{F^{i}_{p}(\varphi)\}_{i=1}^{n}$ contributes to a filtration of multi-GENEO. Then by Definition \ref{mul}, $L^{i}(\varphi)=F^{1,i}_{p}(\varphi)-F^{2,i}_{p}(\varphi)$ for $i=1,...,n$, we could also get the filtration of multi-DGENEO and mix-GENEO.

\section{Experiments}
In this section, we aim to demonstrate the effectiveness of our method in previous  paper. We will use GENEO and DGENEO to extract multiparameter filtration from partial and complete MNIST dataset, and we will use the tool RIVET and multiparameter persistence landscape to represent the rank invariants of the multiparameter persistence module. To construct comparable experiments, we use Dionysus to build lower-star and upper-star filtrations on images and use persistence images to vectorize their persistence diagrams.

RIVET is used to provide the corresponding results in biparameter persistence module. RIVET software can compute and visualize three such kinds of invariants, the Hilbert function, the bigraded Betti numbers, and the fibered barcode. RIVET supports the fast computation of multigraded Betti numbers and an interactive visualisation for 2-parameter persistence modules. RIVET approximates multiparameter modules with a discretization in order to reduce computational cost. These approximations can be taken to arbitrary accuracy with respect to the interleaving distance. Details of the time and space complexity of the algorithm may be found in \cite{zomorodian2004computing, Micheal-Mattew-2015, RIVET-2020}.

Multiparameter persistence landscapes are stable with respect to the interleaving distance and persistence weighted Wasserstein distance which can be found in \url{https://github.com/OliverVipond/Multiparameter_Persistence_Landscapes/}.  \cite{Oli-2020} provided statistical tests to demonstrate their potential applications of landscapes.

Dionysus is written in  C++, with python bindings, which provides various algorithms with clean and consistent internal design for computing persistent homology, which can be founded in \url{https://mrzv.org/software/dionysus2/}.  This package is useful to build  lower-star and upper-star filtrations of the Freudenthal triangulation on a grid.

Persistence Images is a useful  and stable vector representation of persistence diagrams, which  is proposed in \cite{Adams-2017}. Specifically, it  allows users to assign a weight on each point in a persistence diagram, and provide an efficient and easily understandable approach to vectorize persistence diagrams for machine learning tasks. We using Python packages from \url{https://persim.scikit-tda.org/} to compute persistence images.

All experiments were run on a laptop with an AMD Ryzen 7 5800H with Radeon Graphics and 16GB of memory.

\subsection{Generating Bifiltrations on Digital Images}

In this subsection, we will provide an algorithm to generate biparameter filtrations on digital images, which is also suitable for n-parameter filtrations. We give an example to show how generating biparameter filtration on digital images.

There have been several methods to construct cubical complexes.
\cite{Rob-2011} represented the voxels as vertices of the cubical complexes, and then \cite{Bea-2022} used this method to build cubical complexes from an image $\varphi:X\rightarrow \mathbb{R}$.  \cite{Dion-2012} built lower-star and upper-star filtrations of the Freudenthal triangulation on a grid in Dionysus . Inspired by their contributions, we build simplicial complex from two images $\varphi_{1}$, $\varphi_{2}\in{\mathcal{\varPhi}}$ by considering a unit square as two 2-simplices. 

Recall that such a grayscale image is a function $\varphi:X\rightarrow \mathbb{R}$, where $X\subset \mathbb{Z}^{2}$ is typically a rectangular subset of the discrete lattice
$$X=\{(m,n)~|~0\leq m\leq M,0\leq n\leq N\}.$$
A point $(m,n)\in X$ is called a pixel and the value $\varphi(x)\in \mathbb{R}$ is the grayscale value of $x$. The pixels $x\in X$ are the vertices (0-cells) of the complex.
If two vertices whose coordinates differ by one in a single axis, then the edge with endpoints of the two vertices is one 1-simplex. If four vertices form a unit square, then the edge with endpoints located in the upper left and the lower right is also one 1-simplex. And then, the unit square is divided into two 2-simplices. An example is given in Figure \ref{e_f_ex}.

\begin{figure}[h]
	\qquad  
	\begin{minipage}{0.3\linewidth}
		\vspace{3pt}
		\centerline{\includegraphics[height=1.7cm,width=1.7cm]{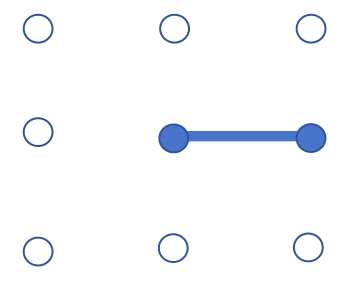}}
	\end{minipage}
	\qquad \quad \quad \ \ 
	\begin{minipage}{0.3\linewidth}
		\vspace{3pt}
		\centerline{\includegraphics[height=1.7cm,width=1.7cm]{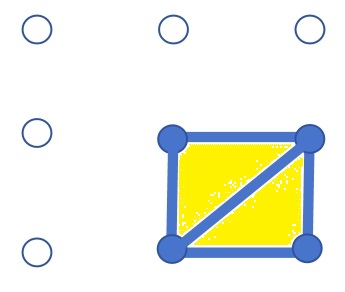}}
		
	\end{minipage}
	
	\caption{The solid dots represent vertices that have already appeared. There is one edge with two endpoints in the left figure and there are two 2-simplices colored in yellow.}
	\label{e_f_ex}
\end{figure}

Suppose that two grayscale digital images $\varphi_{1}$ and $\varphi_{2}$ are represented by the following two matrices 
$$
\begin{bmatrix} 
	7 & 5 & 3 \\
	8 & 6 & 9 \\
	1 & 4 & 2
\end{bmatrix} \ \ \ \ \quad \ \ \text{and}	\ \ \ \ \quad \ \ \ 
\begin{bmatrix} 
	3 & 2 & 7 \\
	4 & 9 & 8 \\
	5 & 6 & 1
\end{bmatrix} .
$$ 
Then we use nine letters from $a$ to $i$ to mark the nine vertices as follows,
$$
\begin{matrix}
	\bullet & \bullet & \bullet \\
	\bullet & \bullet & \bullet \\
	\bullet & \bullet & \bullet
\end{matrix} \qquad \qquad \qquad \qquad \ \ \ \ \ 
\begin{matrix}
	g & h & i \\
	d & e & f \\
	a & b & c
\end{matrix}
$$
By taking sublevelset filtration, a bifiltration could be shown in Figure \ref{fig:myfigure}. The complexes in the position $(p,q)$ are generated by the pixels $x$ which satisfying
$\varphi_{1}(x)\leq p \ \text{and} \  \varphi_{2}(x)\leq q.$   
Notice that the 0-simplices in the position $(4,6)$ are $a$, $b$ and $c$, the 1-simplices are $ab$ and $bc$. The simplices $b$, $ab$ and $bc$  first appear in the position.
Call $(p,q)$ the birth coordinate of them.
\begin{figure}[htbp]
	\centering
	\includegraphics[width=0.5\textwidth]{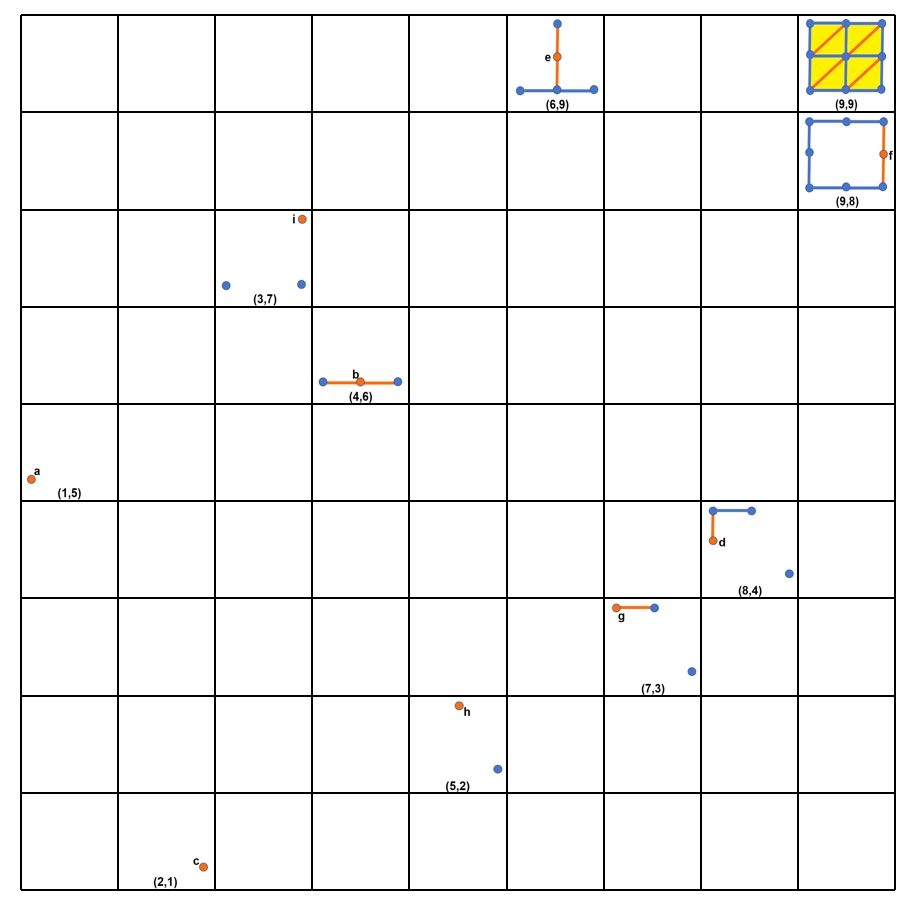}
	\caption{Bifiltration Example. The figure records the birth coordinates of vertices, edges and faces. The vertices and edges in the birth coordinates are colored in orange, the faces in the birth coordinates are colored in yellow, the rest are colored in blue.}
	\label{fig:myfigure}
\end{figure}

Furthermore, we can use RIVET to visualize the biparameter persistence modules of 0-dimensional and 1-dimensional homology in Figure \ref{fig4}. For details of basic persistent homology, we refer to \cite{zomorodian2004computing, zomorodian2005topology, botnan2022introduction}.

\begin{figure}[h]
	\ 
	\begin{minipage}{0.45\linewidth}
		\vspace{3pt}
		\centerline{\includegraphics[height=2.5cm,width=2.5cm]{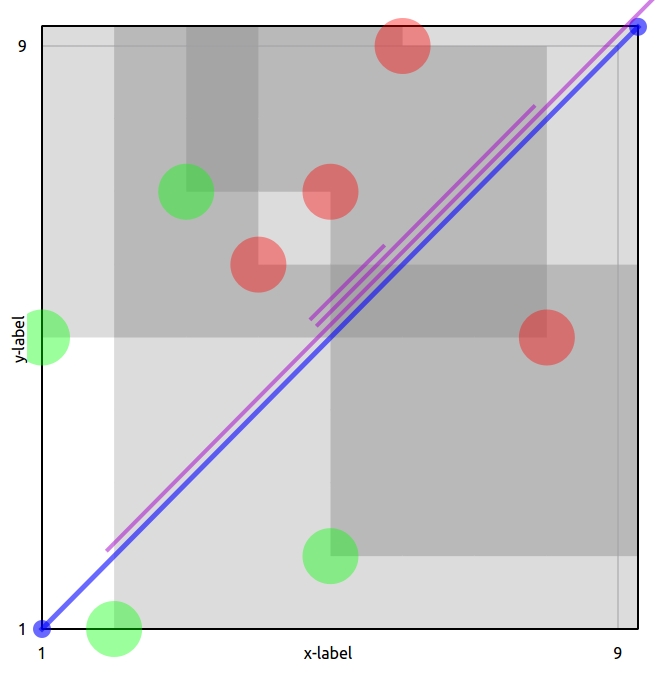}}
		\centerline{$\mathscr{B}_{H(X)^{0}}$}
	\end{minipage}
	\quad 
	\begin{minipage}{0.45\linewidth}
		\vspace{3pt}
		\centerline{\includegraphics[height=2.5cm,width=2.5cm]{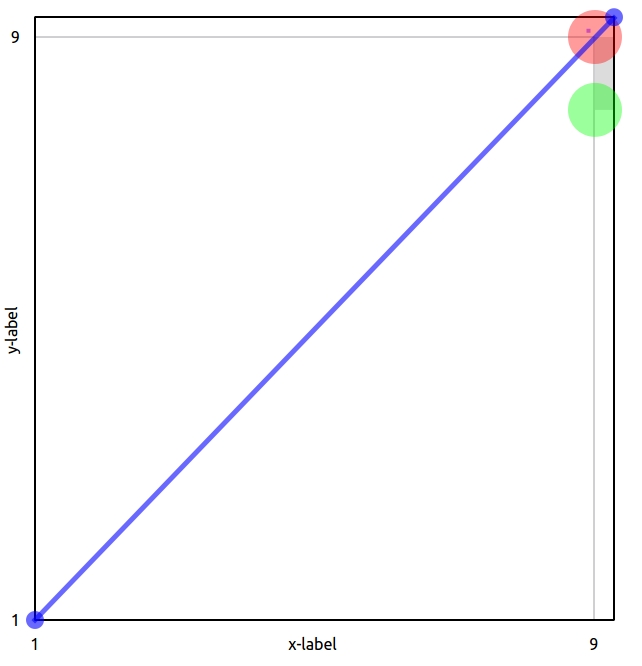}}
		
		\centerline{$\mathscr{B}_{H(X)^{1}}$}
	\end{minipage}
	
	\caption{$\mathscr{B}_{H(X)^{0}}$ reprensents the $H_{0}$ multiparameter persistence module, $\mathscr{B}_{H(X)^{1}}$ reprensents the $H_{0}$ multiparameter persistence module. One can see 1-loop in $\mathscr{B}_{H(X)^{1}}$ only birth at the coordinate (9,8) and persist to the coordinate (9,9).}
	\label{fig4}
\end{figure}

Notice that the bifiltration is a one-critical multifiltration defined in \cite{Gunnar-2010} since each cell of the multifilter complex has a unique critical coordinate.

\begin{algorithm}[h]{\tiny}
	
	\renewcommand{\algorithmicrequire}{\textbf{Input:}}
	\renewcommand{\algorithmicensure}{\textbf{Output:}}
	\caption{Build bifiltration}
	\label{BD}
	\begin{algorithmic}
		\Require $\mathcal{V}$, vertex ; 
		\Require $\varphi$, image;
		\Require $M = (M^{1},M^{2})$, mix-GENEO;
		\Ensure $\mathscr{F}=(\mathscr{F}_{x},\mathscr{F}_{y})$, bifiltration at $(x,y)$; 
		
		\State {$\psi_{1}=M^{1}(\varphi)$, $\psi_{2}=M^{2}(\varphi)$;}	
		\State {$\mathscr{F}\leftarrow$empty;}
		\For {$v\in{\mathcal{V}}$};
		\State {$(\mathscr{F}_{vx},\mathscr{F}_{vy})=(\psi_{1}(v), \psi_{2}(v))$;}
		\State {$\mathscr{F}\leftarrow{\mathscr{F} \bigcup (\mathscr{F}_{vx},\mathscr{F}_{vy})}$;}
		\EndFor
		
		\State {$\mathcal{E}\leftarrow$empty;}
		\If {$v_{i}$ is adjacent to $v_{j}$}
		\State {$\mathcal{E}\leftarrow{\mathcal{E}\bigcup \{e_{ij}\}};$}
		\EndIf
		
		\For {$e_{ij}\in{\mathcal{E}}$};
		\State {$(\mathscr{F}_{e_{ij}x},\mathscr{F}_{e_{ij}y})$\\
			$=(\max(\mathscr{F}_{v_{i}x},\mathscr{F}_{v_{j}x}), \max(\mathscr{F}_{v_{i}y},\mathscr{F}_{v_{j}y}));$}
		\State {$\mathscr{F}\leftarrow{\mathscr{F} \bigcup (\mathscr{F}_{e_{ij}x},\mathscr{F}_{e_{ij}y})}$;}
		\EndFor
		
		\State {$\mathcal{F}\leftarrow$empty;}
		\If{four vertices $v_{i}$, $v_{j}$, $v_{k}$, $v_{s}$ form a square, and $e_{ik}$ is a diagonal line in the square with a fixed direction}
		\State {$\mathcal{F}\leftarrow{\mathcal{F}\bigcup f_{ijk}};$}
		\State {$\mathcal{F}\leftarrow{\mathcal{F}\bigcup f_{isk}};$}
		\EndIf
		
		\For {$f_{ijk}, f_{isk}\in{\mathcal{F}}$};
		\State {
			$(\mathcal{F}_{f_{ijk}x},\mathcal{F}_{f_{ijk}y})$\\$= \ \  (\max(\mathscr{F}_{v_{i}x},\mathscr{F}_{v_{j}x},\mathscr{F}_{v_{k}x},\mathscr{F}_{v_{s}x}$),}
		\State{ $\max(\mathscr{F}_{v_{i}y},\mathscr{F}_{v_{j}y},\mathscr{F}_{v_{k}y},\mathscr{F}_{v_{s}y}));$}
		\State
		{ $(\mathcal{F}_{f_{isk}x},\mathcal{F}_{f_{isk}y})$\\$= \ \ (\max(\mathscr{F}_{v_{i}x},\mathscr{F}_{v_{j}x},\mathscr{F}_{v_{k}x},\mathscr{F}_{v_{s}x}$),}
		\State{
			$\max(\mathscr{F}_{v_{i}y},\mathscr{F}_{v_{j}y},\mathscr{F}_{v_{k}y},\mathscr{F}_{v_{s}y}));$}
		\State {$\mathscr{F}\leftarrow{\mathscr{F} \bigcup (\mathscr{F}_{f_{ijk}x},\mathscr{F}_{f_{ijk}y}) \bigcup (\mathscr{F}_{f_{isk}x},\mathscr{F}_{f_{isk}y})}$;}				
		\EndFor

		\State \Return $\mathscr{F}$.
	\end{algorithmic}
\end{algorithm}

$\textbf{Complexity}$ We now explore the complexity of Algorithm \ref{BD}. Notice that the bifiltration we construct is all one-critical. One vertex is computed one time if it is seemed as 0-simplex or a vertex of a higher simplex. A vertex is a common vertex of at most six 1-simplices and six 2-simplices. The algorithm requires at most $O(13n)$ time for n vertices.

\subsection{Example Computations}

In this subsection, we will provide examples of calculating binary classification and ten-classification. And we compare the performance of lower-star filtration, upper-star filtration, multi-GENEO, multi-DGENEO and mix-GENEO in the classification tasks by vectorizations. We consider two types of binary classifications, one is studied by taking 500 samples from the MNIST dataset, the other one is studied by using the complete MNIST dataset.

\subsubsection{Comparison of 1-parameter filtrations, multi-GENEO, multi-DGENEO and mix-GENEO using binary classification}

We will give examples of multi-GENEO, multi-DGENEO and mix-GENEO persistence filtrations to validate the effectiveness of our multifiltrations on MNIST dataset. We compare the performances of lower-star filtration, upper-star filtration, multi-GENEO, multi-DGENEO and mix-GENEO persistence filtrations for binary classification and ten-classification. 1-parameter filtrations are vectorized by persistence images \cite{Adams-2017}, and 2-parameter filtrations are vectorized by multiparameter persistence landscapes \cite{Oli-2020}. One can see that mix-GENEO performs the best on partial MNIST dataset.


Suppose that a digital image is the bounded function $\varphi$. We select five GENEOs, $G_{0}$, $G_{1}$, $G_{2}$, $G_{3}$ and $G_{4}$, to get bifiltration $\{F_{p}^{i}(\varphi)\}_{i=1}^{2}$. Notice that $G_{0}$ can be seemed as a Gaussian blur, $G_{1}-G_{2}$ and $G_{3}-G_{4}$ which are called DOG can be seemed as Laplace operators approximately. Since identity $I$ is also a GENEO, we could build multi-GENEO filtration by $G_{0}$ and $I$ acting on $\varphi$. Multi-DGENEO fitration is built by $(G_{3}-G_{4})(\varphi)$ and $(G_{1}-G_{2})(\varphi)$, and mix-GENEO filtration is built by $G_{0}(\varphi)$ and $(G_{3}-G_{4})(\varphi)$. To make the parameters in RIVET and persistent landscape consistency, we resize the value of $F_{p}^{i}(\varphi)$ into $[0,255]$.

Considering the images of the numbers $\{0,1,3,6,9\}$, we perform 500 samples for each number according to the order of appearance in the MNIST dataset. For 1-parameter filtrations, we use Dionysus to build lower-star filtration and upper-star filtration from these samples.   For 2-parameter filtrations, we use RIVET to build our three multifiltrations. To make the operation faster, we use the parameter bin in RIVET equal to 10  which coarsen persistence module to obtain an algebraically simpler module.

For 1-parameter filtraions, we set the resolution of persistence image to be 5, the Gaussian sigma  to be 1 and the persistence range to be $(0,256)$. As well known, the persistence images is based on barcodes which the death time is greater than the birth time. For the barcode generated by upper-star filtration, the birth time is later than the death time, so we swapped the birth time and death time of the barcodes generated by upper filtration. For 2-parameter filtrations, we plot the average persistence landscape $\lambda(k,\boldsymbol{x})$ for $k=1$ in the parameter range $[0,255]^{2}$ of the five datasets and the complete MNIST dataset with stepsize $s=10$ for the $H_{0}$-modules and $H_{1}$-modules (See Figures \ref{fig:landh0}, \ref{fig:landh1},  \ref{fig:landmnisth0} and \ref{fig:landmnisth1}). Here the first landscape $\lambda(k,\boldsymbol{x})$ detects the parameter values for which the associated space has at least 1-homological features together with the persistence of those features. 

Figures \ref{fig:landh0}, \ref{fig:landh1},  \ref{fig:landmnisth0} and \ref{fig:landmnisth1} show that the $H_{1}$ of number $1$  is significantly different from numbers $s\in \{0,3,6,9\}$ since $1$ has different topological and geometric information. It is worthy to note that although the topological and geometric information of $6$ and $9$ are almost the same, we can also find significant differences between them.

All landscapes of numbers from 0 to 9 can be found in our github code.

\begin{figure}[htbp]
	\centering
	\includegraphics[width=0.5\textwidth]{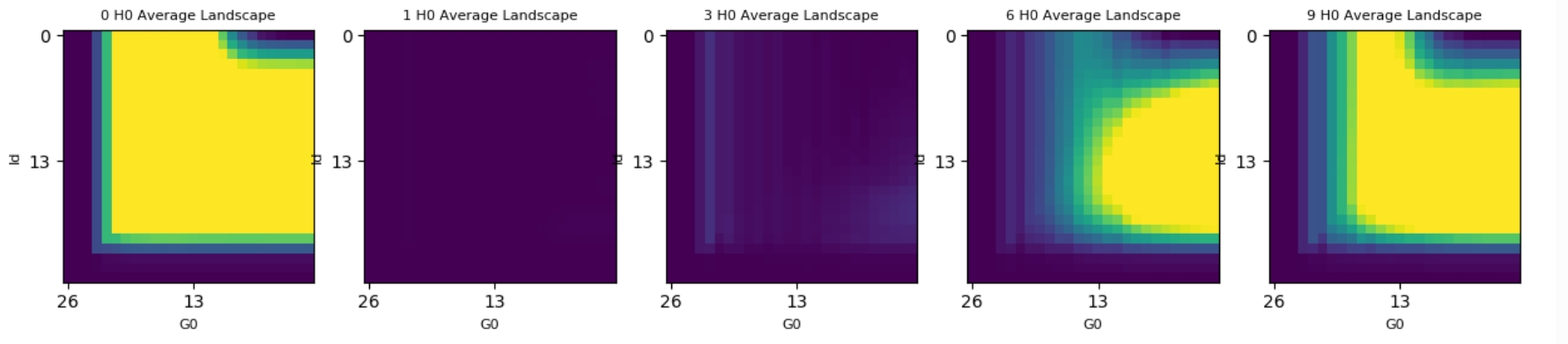}
	\caption{Multi-GENEO: Average Multiparameter Persistence Landscape for each number in $\{0,1,3,6,9\}$ by taking 500 samples from MNIST dataset ($H_{0}$).}
	\label{fig:landh0}
\end{figure}
\begin{figure}[htbp]
	\centering
	\includegraphics[width=0.5\textwidth]{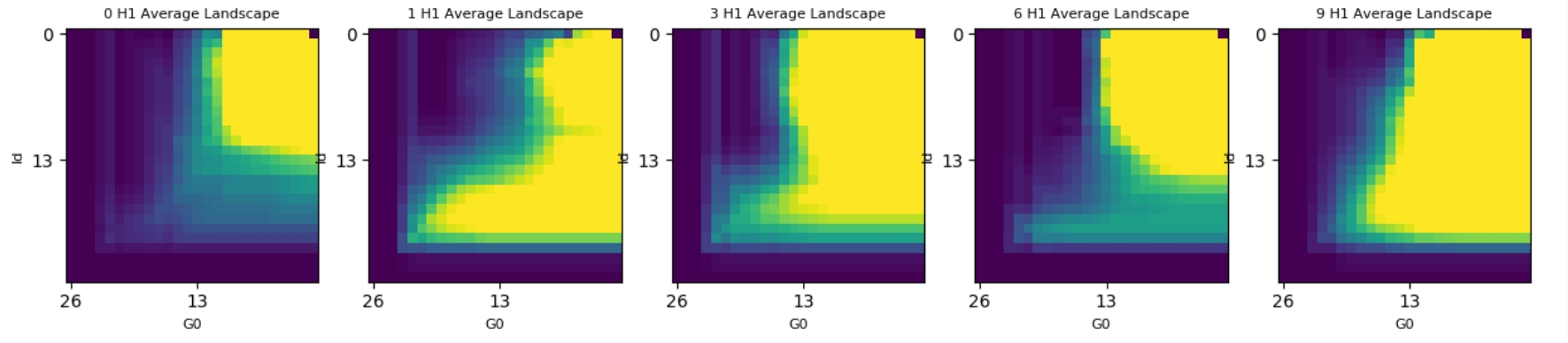}
	\caption{Multi-GENEO: Average Multiparameter Persistence Landscape for each number in $\{0,1,3,6,9\}$ by taking 500 samples from MNIST dataset ($H_{1}$).}
	\label{fig:landh1}
\end{figure}

\begin{figure}[htbp]
	\centering
	\includegraphics[width=0.5\textwidth]{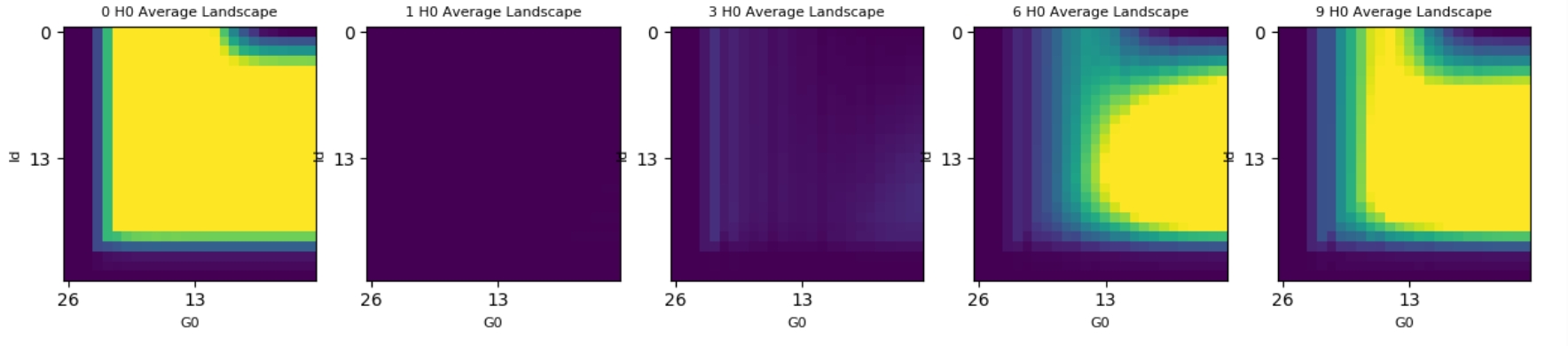}
	\caption{Multi-GENEO: Average Multiparameter Persistence Landscape for each number in $\{0,1,3,6,9\}$ in MNIST dataset ($H_{0}$).}
	\label{fig:landmnisth0}
\end{figure}
\begin{figure}[htbp]
	\centering
	\includegraphics[width=0.5\textwidth]{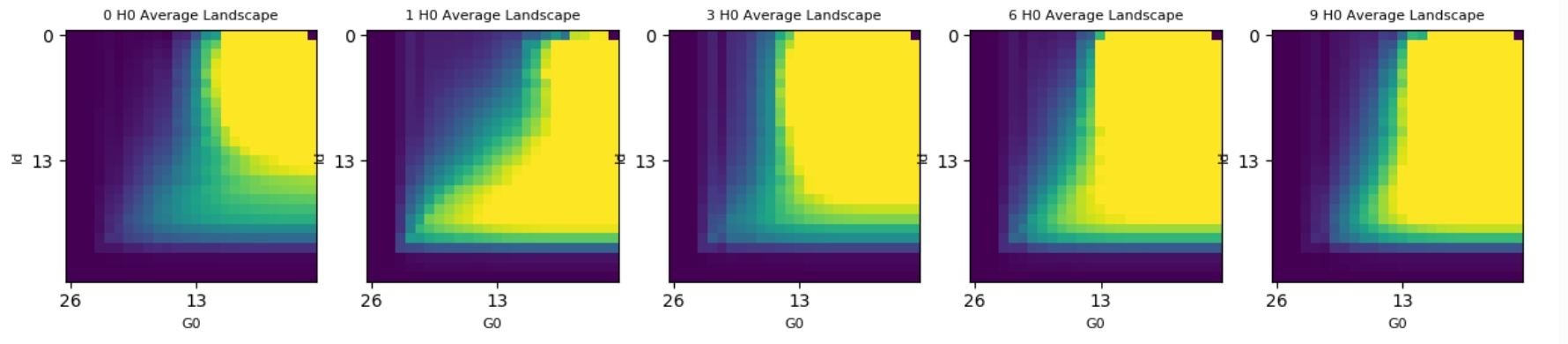}
	\caption{Multi-GENEO: Average Multiparameter Persistence Landscape for each number in $\{0,1,3,6,9\}$ in MNIST dataset ($H_{0}$).}
	\label{fig:landmnisth1}
\end{figure}

We use machine learning algorithms with multiparameter landscape functions and persistence images as a collection of features for a data set to learn non-linear relationships in our data set, and then we apply Principle Components Analysis (PCA) to the collection of $H_{0}$, $H_{1}$ and $(H_{0},H_{1})$ landscape vectors $\lambda(1,\boldsymbol{x})$ as well as persistence images to reduce parameter dimensions. The PCA projections make our methods better to verify the topological and geometric information of  the digital  dataset.

$\textbf{Dataset}$
The MNIST dataset is a classic dataset in the field of machine learning, consisting of 60000 training samples and 10000 test samples, each of which is a  $28\times28$  pixel grayscale handwritten digital image and represents a number from 0 to 9.

$\textbf{Results}$
We obtain the accuracies of binary classifications of 0 and 1 , 1 and 3 , 6 and 9 by persistent images of lower-star filtration and upper-star filtration, and by multiparamter persistent landscapes of multi-GENEO, multi-DGENEO and mix-GENEO, respectively. For 500 samples of each number, we perform 100 trials and average the classification accuracies. For the complete MNIST dataset, we use their train and test datasets for training and testing, respectively. More details of the results are provided in Table \ref{table1} and Table \ref{table2}. The accuracy of mix-GENEO of binary classification of 0 and 1, which have different topological information, can achieve $99.5\%$. The accuracy of mix-GENEO of binary classification of 6 and 9, which have almost the same topological and geometric information, can achieve $94.1\%$. The accuracy of mix-GENEO of binary classification of 1 and 3, which have different geometric information, can achieve $99.1\%$. Mix-GENEO is superior to 1-parameter filtrations. Therefore, our methods can significantly distinguish not only the ones with different topological information but also the ones with almost the same topological and geometric information. In our three methods,  multi-GENEO is suitable for $H_{0}$, multi-DGENEO is suitable for $H_{1}$ and mix-GENEO performs well for both $H_{0}$ and $H_{1}$. In general, mix-GENEO performs the best. In particular, persistence diagrams generated by lower-star filtration of the numbers $\{1,3\}$ are almost only have the trivial messages, which are $H_{0}$ $(0,+\infty]$ and $H_{1}$ $(0,256]$. 1-parameter filtrations cannot get enough signatures , but our methods make sense. 

\begin{table}[h]\tiny
	\centering
	
	\tabcolsep=0.1cm
	\begin{tabular}{@{}ccccccccccc@{}}

		\hline

		\multicolumn{2}{c}{\multirow{2}{*}{}}
		&\multicolumn{3}{c}{$H_{0}$}&\multicolumn{3}{c}{$H_{1}$}
		&\multicolumn{3}{c}{$H_{0}+H_{1}$}\\
		\cline{3-11}

		\multicolumn{2}{c}{}&  L & PL &PS &  L & PL &PS&  L & PL &PS\\
		\hline
		
		\multirow{5}{*}{$0vs1$} &lower-star & {92.0} & {91.5} & {95.8} & {61.7} & {61.9} & {67.7} & {84.0} & {81.7} & {93.0}\\
		\cline{2-11}
		
		\multirow{5}{*}{} &upper-star & {62.1} & {60.9} & {67.7} & {91.8} & {91.1} & {95.7} & {84.2} & {81.1} & {92.6}\\
		\cline{2-11}
		
		\multirow{5}{*}{} &mul-G & \textbf{97.9} & \textbf{98.1} & {98.6} & {56.6} & {55.7} & {57.4} & {97.3} & {96.5} & {97.6}\\
		\cline{2-11}
		
		\multirow{5}{*}{} &mul-D & {50.1} & {48.3} & {48.1} & {92} & {97.2} & {90.4} & {83.9} & {85.7} & {86.5}\\
		\cline{2-11}
		
		\multirow{5}{*}{} &mix-G & {97.5} & {96.8} & \textbf{99.1} & \textbf{98.5} & \textbf{98.1} & \textbf{98.7} & \textbf{98.2} & \textbf{99.1} & \textbf{99.3}\\
		\hline
		
		\multirow{5}{*}{$1vs3$} &lower-star & {57.7} & {57.2} & {59.1} & {63.5} & {63.5} & {66.9} & {64.4} & {61.9} & {68.2}\\
		\cline{2-11}
		
		\multirow{5}{*}{} &upper-star & {63.9} & {63.1} & {66.6} & {57.4} & {57.0} & {59.0} & {65.0} & {61.9} & {67.0}\\
		\cline{2-11}
		
		\multirow{5}{*}{} &mul-G & {68.9} & {68.5} & {70.4} & {63.4} & {63.9} & {66.2} & {71.9} & {73.2} & {74.8}\\
		\cline{2-11}
		
		\multirow{5}{*}{} &mul-D & {50.3} & {50.1} & {50.3} & {89.6} & {89.8} & {85.4} & {81.4} & {81.8} & {82.5}\\
		\cline{2-11}
		
		\multirow{5}{*}{} &mix-G & \textbf{93.3} & \textbf{92.6} & \textbf{94.3} & \textbf{95.7} & \textbf{95.7} & \textbf{96.8} & \textbf{95.7} & \textbf{96.9} & \textbf{97.6}\\
		\hline
		
		\multirow{5}{*}{$6vs9$} &lower-star & {52.7} & {50.5} & {56.2} & {49.5} & {49.4} & {52.5} & {52.5} & {51.6} & {56.0}\\
		\cline{2-11}
		
		\multirow{5}{*}{} &upper-star & {49.7} & {50.2} & {53.2} & {52.3} & {51.2} & {55.9} & {51.3} & {51.8} & {56.6}\\
		\cline{2-11}
		
		\multirow{5}{*}{} &mul-G & {64.8} & {61.2} & {69.8} & {52.4} & {56.5} & {51.5} & {63.5} & {68.2} & {69.7}\\
		\cline{2-11}
		
		\multirow{5}{*}{} &mul-D & {50} & {48.6} & {48.4} & {69.2} & \textbf{85.3} & \textbf{86.4} & {76.6} & {86.2} & {85.9}\\
		\cline{2-11}
		
		\multirow{5}{*}{} &mix-G & \textbf{76.6} & \textbf{79.7} & \textbf{75.7} & \textbf{73.3} & {80.4} & {82.1} & \textbf{85.8} & \textbf{87.3} & \textbf{88.7}\\
		\hline
	\end{tabular}
	\vspace{2mm}
	\caption{Binary classification of lower-star, upper-star, multi-GENEO, multi-DGENEO and mix-GENEO on '0~vs~1', '1~vs~3' and '6~vs~9' with each number of 500 examples using LDA, PCA+LDA, PCA+SVM.  In the first row, the following abbreviations are used: L=LDA, PL=PCA+LDA, PS=PCA+SVM. Bold indicates highest scores.}
	\label{table1}
\end{table}

\begin{table}[h]\tiny
	\centering
	
	\tabcolsep=0.1cm
	\begin{tabular}{@{}ccccccccccc@{}}
		
		\hline
		\multicolumn{2}{c}{\multirow{2}{*}{}}
		&\multicolumn{3}{c}{$H_{0}$}&\multicolumn{3}{c}{$H_{1}$}
		&\multicolumn{3}{c}{$H_{0}+H_{1}$}\\
		\cline{3-11}

		\multicolumn{2}{c}{}&  L & PL &PS &  L & PL &PS&  L & PL &PS\\
		\hline
		
		\multirow{5}{*}{$0vs1$} &lower-star & {93.7} & {94.9} & {96.2} & {71.9} & {73.6} & {72.1} & {92.3} & {95.2} & {95.7}\\
		\cline{2-11}
		
		\multirow{5}{*}{} &upper-star & {71.9} & {73.5} & {71.8} & {93.7} & {94.9} & \textbf{96.2} & {92.3} & {95.1} & {95.7}\\
		\cline{2-11}
		
		\multirow{5}{*}{} &mul-G & \textbf{99.1} & {94.9} & {98.8} & {63.3} & {63} & {63.9} & {99.1} & {97.3} & {98.7}\\
		\cline{2-11}
		
		\multirow{5}{*}{} &mul-D & {57.7} & {60.3} & {60.2} & {87.8} & {95.6} & {87.5} & {88.8} & {95.6} & {87.8}\\
		\cline{2-11}
		
		\multirow{5}{*}{} &mix-G & \textbf{99.1} & \textbf{98.3} & \textbf{99.3} & \textbf{96.2} & \textbf{96.2} & \textbf{96.2} & \textbf{99.5} & \textbf{99.2} & \textbf{99.4}\\
		
		\hline
		
		\multirow{5}{*}{$1vs3$} &lower-star & {63.0} & {62.1} & {62.0} & {71.8} & {71.8} & {70.7} & {72.8} & {73.4} & {73.1}\\
		\cline{2-11}
		
		\multirow{5}{*}{} &upper-star & {72.0} & {71.8} & {70.5} & {63.0} & {62.1} & {62.0} & {72.9} & {73.0} & {73.3}\\
		\cline{2-11}
		
		\multirow{5}{*}{} &mul-G & {72.6} & {70.0} & {72.3} & {68.2} & {67.4} & {67.6} & {69.4} & {66.9} & {70.3}\\
		\cline{2-11}
		
		\multirow{5}{*}{} &mul-D & {59.7} & {58.4} & {59.4} & {73.4} & {74.5} & {73.6} & {63.7} & {68.3} & {66.5}\\
		\cline{2-11}
		
		\multirow{5}{*}{} &mix-G & \textbf{95.8} & \textbf{95.1} & \textbf{96.2} & \textbf{84.3} & \textbf{82.6} & \textbf{84.2} & \textbf{98.7} & \textbf{98.3} & \textbf{99.1}\\
		
		\hline
		
		\multirow{5}{*}{$6vs9$} &lower-star & {56.7} & {57.8} & {57.8} & {52.6} & {53.4} & {52.1} & {55.5} & {59.0} & {60.7}\\
		\cline{2-11}
		
		\multirow{5}{*}{} &upper-star & {52.3} & {52.6} & {51.7} & {56.7} & {57.8} & {57.8} & {55.1} & {58.4} & {59.3}\\
		\cline{2-11}
		
		\multirow{5}{*}{} &mul-G & {67.8} & {67.1} & {68.6} & {53.9} & {52.9} & {52.4} & {78.5} & {75.7} & {79.0}\\
		\cline{2-11}
		
		\multirow{4}{*}{} &mul-D & {52.1} & {51.7} & {54.3} & {67.1} & \textbf{68.0} & \textbf{67.0} & {75.5} & {76.7} & {75.2}\\
		\cline{2-11}
		
		\multirow{4}{*}{} &mix-G & \textbf{82.9} & \textbf{76.8} & \textbf{83.9} & \textbf{68.3} & {66.2} & {66.6} & \textbf{93.4} & \textbf{91.7} & \textbf{94.1}\\
		
		\hline
	\end{tabular}
	\vspace{2mm}
	\caption{Binary classification results of lower-star, upper-star, multi-GENEO, multi-DGENEO and mix-GENEO for MNIST dataset using LDA, PCA+LDA, PCA+SVM. In the first row, the following abbreviations are used: L=LDA, PL=PCA+LDA, PS=PCA+SVM. Bold indicates highest scores.}
	\label{table2}
\end{table}

\newpage
\subsubsection{Comparison of 1-parameter filtrations, multi-GENEO, multi-DGENEO and mix-GENEO using ten-classification}
For the completeness of the experiment, we also carry out the same experiments on the entire MNIST dataset. 


$\textbf{Results}$ The accuracies of ten-classification are shown in Table \ref{table3}. One can see mix-GENEO performs best, it can effectively identify ten classes and achieve an accuracy of $78.8\%$. 

\begin{table}[h]\tiny
	\centering
	
	\tabcolsep=0.1cm
	\begin{tabular}{@{}cccccccccc@{}}

		\hline
		\multicolumn{1}{c}{\multirow{1}{*}{}}
		&\multicolumn{3}{c}{$H_{0}$}&\multicolumn{3}{c}{$H_{1}$}
		&\multicolumn{3}{c}{$H_{0}+H_{1}$}\\
		\cline{2-10}

		\multicolumn{1}{c}{}&  L & PL & PS &  L & PL & PS &  L & PL & PS \\
		\cline{1-10}
		
		\multicolumn{0}{c}{lower-star}  & {31.2} & {31.4} &{29.9} & {19.3} & {19.3} &{19.0} & {35.6} & {35.3} &{34.0}  \\
		\cline{1-10}
		
		\multicolumn{0}{c}{upper-star}  & {18.6} & {18.7} &{18.3} & \textbf{31.2} & {31.4} &{29.9} & {35.6} & {35.3} &{34.0}  \\
		\cline{1-10}
		
		\multicolumn{0}{c}{mul-G}  & \textbf{39.4} & {39.6} &{39.7} & {19.1} & {19.1} &{19.3} & \textbf{42.7} & {42.9} &{43.4}  \\
		\cline{1-10} 
		
		\multicolumn{0}{c}{mul-D}  & {9.8} & {14.2} &{14.2} & {30.3} & {32.1} &{29.4} & {9.8} & {33.6} &{31} \\
		\cline{1-10}
		
		\multicolumn{0}{c}{mix-G}  & {9.8} & \textbf{64.4} &\textbf{67.8} & {19.2} & \textbf{46.7} &\textbf{50.6} & {11.3} & \textbf{73} &\textbf{78.8} \\
		\cline{1-10}
		
	\end{tabular}
	\vspace{2mm}
	\caption{Ten-classification results of lower-star, upper-star, multi-GENEO, multi-DGENEO and mix-GENEO for MNIST dataset using LDA, PCA+LDA, PCA+SVM. In the first row, the following abbreviations are used: L=LDA, PL=PCA+LDA, PS=PCA+SVM. Bold indicates highest scores.}
	\label{table3}
\end{table}

\section{Conclusion and future work}
In this paper, we introduce three multiparameter persistence filtrations  called multi-GENEO, multi-DGENEO and mix-GENEO which can be chosen flexible. Moreover, we show the stability of both interleaving distance and multiparameter persistence landscape of multi-GENEO persistence module. We also provide estimations of upper bound for multi-DGENEO and mix-GENEO persistence module with respect to pseudometrics. After giving an algorithm to build the bifiltrations on digital images, the experiments we conduct demonstrate that our methods perform better 1-parameter filtrations, and demonstrate that our methods can significantly distinguish not only the ones with different topological information but also the ones with almost the same topological and geometric information.

In the future work, we would like to develop our methods in the following two aspects. On the one hand, we plan to optimatize our methods to get better results. For instance, we would obtain multiparameter filtrations by higher dimensional sublevelset functions or by selecting suitable operators in another way. On the other hand, we plan to apply our methods to other fields or problems, for instance, integrating features into deep learning and medical research. 




\backmatter

%
%
%

\bmhead{Acknowledgements}

We thank the editor and the reviewers for their dedicated time and constructive comments.

\section*{Declarations}


\begin{itemize}
\item Funding \\ This work was partially supported by the National Key R\&D Program of China (No. 2020YFA0714101).
\item Conflict of interest/Competing interests \\
The corresponding author declares that there are no conflicts of interest in this work.
\item Code availability \\
Our code is available at \url{https://github.com/HeJiaxing-hjx/Mix-GENEO/}.
\item Author contribution\\
These authors contributed equally to this work.
\end{itemize}

\bibliography{reference_abbr}

\end{document}